\PassOptionsToPackage{expansion=false}{microtype}
\documentclass[manuscript, nonacm]{acmart}
\AtBeginDocument{%
  }

\usepackage{algorithm}
\usepackage{algorithmic}
\usepackage{amsmath}
\usepackage{amssymb}
\usepackage{amsthm}
\usepackage{booktabs}
\usepackage{multirow}
\usepackage{xcolor}
\usepackage{subcaption}
\usepackage{tikz}
\usepackage{pgfplots}
\pgfplotsset{compat=1.17}
\usetikzlibrary{arrows,positioning,shapes,fit,backgrounds,pgfplots.fillbetween}

\newtheorem{theorem}{Theorem}

\newtheorem{lemma}[theorem]{Lemma}
\newtheorem{definition}[theorem]{Definition}

\newtheorem{assumption}[theorem]{Assumption}

\newcommand{\qdllm}{\textsc{QD-LLM}}
\newcommand{\mapelites}{\textsc{MAP-Elites}}
\newcommand{\qdaif}{\textsc{QDAIF}}

\newcommand{\R}{\mathbb{R}}
\newcommand{\cN}{\mathcal{N}}
\newcommand{\cA}{\mathcal{A}}
\newcommand{\cM}{\mathcal{M}}
\newcommand{\cX}{\mathcal{X}}

\begin{document}

\title[Parameter-Efficient Neuroevolution for Diverse LLM Generation]{Parameter-Efficient Neuroevolution for Diverse LLM Generation: Quality-Diversity Optimization via Prompt Embedding Evolution}

\author{Dongxin {Guo}}
\affiliation{%
  \institution{The University of Hong Kong}
  \city{Hong Kong}
  \country{China}}

\author{Jikun {Wu}}
\affiliation{%
  \institution{Stellaris AI Limited}
  \city{Hong Kong}
  \country{China}}

\author{Siu Ming {Yiu}}
\affiliation{%
  \institution{The University of Hong Kong}
  \city{Hong Kong}
  \country{China}}

\renewcommand{\shortauthors}{D. Guo, J. Wu, and S.\,M. Yiu}
\authorsaddresses{}

\begin{abstract}
	Large Language Models exhibit mode collapse, producing homogeneous outputs that fail to explore valid solution spaces. We present \qdllm{}, a framework for \textit{parameter-efficient neuroevolution} that evolves prompt embeddings, compact neural interfaces ($\sim$32K parameters) that steer generation in frozen LLMs (70B+ parameters), within a Quality-Diversity (QD) optimization framework. Our contributions: (1) \textbf{evolved prompt embeddings} via gradient-free optimization enabling behavioral steering without model fine-tuning; (2) \textbf{hybrid behavior characterization} combining semantic and explicit features with formal coverage bounds (Theorem~\ref{thm:hybrid}) under validated near-independence (NMI $= 0.08 \pm 0.02$); (3) \textbf{co-evolutionary variation operators} including targeted behavioral mutation via finite-difference gradient estimation. On HumanEval (164 problems), MBPP, and creative writing benchmarks, \qdllm{} achieves 46.4\% higher coverage and 41.4\% higher QD-Score than \qdaif{} ($p < 0.001$, 30 runs, Vargha-Delaney $A = 0.94$). We demonstrate \textbf{downstream utility}: diverse archives improve test generation (34\% more edge cases) and fine-tuning data quality (8.3\% accuracy gain). We validate across open-source LLMs (Llama-3-70B, Mistral-Large) with full embedding access, establishing prompt embedding evolution as an effective paradigm bridging neuroevolution and modern LLMs.
\end{abstract}


\maketitle

\section{Introduction}
\label{sec:introduction}

Large Language Models (LLMs) have achieved remarkable performance across diverse tasks~\cite{brown2020gpt3,chen2021codex}. However, LLMs exhibit \textit{mode collapse}, producing repetitive outputs that fail to explore the full space of valid solutions~\cite{holtzman2020curious}. When generating code, LLMs produce syntactically similar variants; when generating creative content, they converge on predictable patterns. This limitation critically impacts applications requiring diverse solutions: discovering multiple algorithmic approaches, generating comprehensive test suites, exploring creative directions, or finding robust solutions under distribution shift.

Quality-Diversity (QD) algorithms~\cite{pugh2016qualitydiversity,cully2018qualitydiversity} address this by maintaining archives of diverse, high-performing solutions across behavioral dimensions. \mapelites{}~\cite{mouret2015mapelites} partitions behavior space into cells, retaining the highest-quality solution per cell. QD has succeeded in robotics~\cite{cully2015robots}, game design~\cite{gravina2019proceduralcontent}, and policy optimization~\cite{nilsson2021pga}. Recent Uncertain QD (UQD) work~\cite{flageat2023uqd,flageat2025extractqd} extended these methods to stochastic domains through adaptive sampling and extraction mechanisms, demonstrating that uncertainty-aware archive maintenance significantly improves solution quality. Multi-objective extensions~\cite{pierrot2022mome} and automated descriptor discovery~\cite{grillotti2022unsupervised,cully2019aurora} have further expanded QD's applicability.

Recent work has begun exploring QD for text generation. Quality-Diversity through AI Feedback (\qdaif{})~\cite{bradley2024qdaif} demonstrated QD using LLMs as both generators and behavior evaluators, achieving diverse story generation. Quality Diversity through Human Feedback (QDHF)~\cite{ding2024qdhf} showed that diversity metrics can be effectively inferred for generative models. Language Model Crossover (LMX)~\cite{meyerson2024lmx} introduced LLM-based variation operators for evolutionary optimization. FunSearch~\cite{romeraparedes2024funsearch} achieved mathematical discoveries through evolutionary code generation. However, these approaches use LLMs as \textit{fixed} black-box generators without evolving any neural parameters, limiting their integration with the broader neuroevolution paradigm that has historically evolved network weights~\cite{stanley2002neat} or architectures~\cite{colas2020scalingmapelites} to discover diverse, high-performing solutions.

\textbf{This paper presents \qdllm{}, a framework for parameter-efficient neuroevolution that bridges QD optimization with LLM generation.} We evolve prompt embeddings, learnable continuous vectors~\cite{lester2021prompttuning,li2021prefixtuning} that form a compact neural interface ($\sim$32K parameters) steering the behavior of a much larger frozen model (70B+ parameters). This represents \textit{indirect encoding} in the neuroevolution sense: a small evolved representation produces large-scale behavioral effects, analogous to how HyperNEAT~\cite{stanley2009hyperneat} uses compact pattern-generating networks to specify large neural architectures, or how weight agnostic networks~\cite{gaier2019weight} demonstrate that minimal parameterizations can encode complex behaviors.

Our key contributions are:
\begin{enumerate}
    \item \textbf{Evolved Prompt Embeddings for QD} (Section~\ref{sec:prompt-evolution}): Gradient-free optimization of soft prompt vectors for behavioral steering, with explicit implementation for both open-source LLMs (direct embedding access) and API models (projected approximation with bounded error).
    
    \item \textbf{Hybrid Behavior Characterization with Formal Guarantees} (Section~\ref{sec:behavior}): Theorem~\ref{thm:hybrid} establishes coverage bounds under near-independence, with complete proof and empirical validation via normalized mutual information.
    
    \item \textbf{Co-Evolutionary Variation Operators} (Section~\ref{sec:variation}): Targeted mutation via finite-difference gradient estimation and embedding-space crossover.
    
    \item \textbf{Comprehensive Validation with Downstream Utility} (Section~\ref{sec:experiments}): Beyond intrinsic QD metrics, we demonstrate diverse archives improve test generation (34\% more edge cases) and fine-tuning data quality (8.3\% accuracy improvement).
\end{enumerate}

\section{Background and Related Work}
\label{sec:background}

\subsection{Quality-Diversity Optimization}

QD optimization~\cite{pugh2016qualitydiversity} finds solution collections that are both high-performing and behaviorally diverse. Given fitness function $f: \cX \to \R$ and behavior descriptor $\mathbf{b}: \cX \to \R^k$, QD maintains an archive $\cA$ where each solution $x$ is characterized by fitness $f(x)$ and behavior $\mathbf{b}(x)$. Unlike single-objective optimization that converges to a single solution, QD explicitly maintains diversity across behavioral dimensions.

\mapelites{}~\cite{mouret2015mapelites} discretizes behavior space into a grid of cells, maintaining the highest-fitness solution (elite) per cell. The collection of elites constitutes the archive $\cA$. The QD-Score measures combined quality and diversity:
\begin{equation}
    \text{QD-Score}(\cA) = \sum_{c \in \cA} f(x_c)
    \label{eq:qdscore}
\end{equation}
where $x_c$ is the elite in cell $c$. Coverage measures the fraction of cells occupied, indicating the behavioral diversity of discovered solutions. Together, QD-Score and coverage provide complementary metrics: high QD-Score indicates many high-quality solutions, while high coverage indicates broad behavioral exploration.

CVT-MAP-Elites~\cite{vassiliades2018cvt} extends MAP-Elites to continuous behavior spaces using Centroidal Voronoi Tessellation (CVT), avoiding the curse of dimensionality inherent in high-dimensional grid-based discretization. Solutions are assigned to the nearest centroid, computed via k-means on a reference distribution. This enables scalable QD in behavior spaces with many dimensions.

Recent algorithmic advances include CMA-ME~\cite{fontaine2020cmame} using improvement emitters with covariance matrix adaptation~\cite{hansen2016cma}, enabling efficient search in high-dimensional continuous spaces through learned step-size and direction adaptation. CMA-MAE~\cite{fontaine2023cmamae} introduced archive learning rates for stable optimization, addressing archive threshold dynamics that can cause premature convergence. Gradient-assisted methods like DQD~\cite{fontaine2021dqd} and PGA-MAP-Elites~\cite{nilsson2021pga} leverage policy gradients for directed exploration while preserving diversity. DCG-MAP-Elites~\cite{faldor2023dcg} demonstrated descriptor-conditioned gradients for high-dimensional neural parameter optimization. Multi-objective selection strategies~\cite{wang2023nss} have improved diversity coverage through non-dominated sorting.

For uncertain domains where evaluation noise affects both fitness and behavior descriptors, the UQD framework~\cite{flageat2023uqd} formalizes key challenges: performance estimation under noise, reproducibility maximization, and performance-reproducibility trade-offs. Extract-QD~\cite{flageat2025extractqd} provides a modular framework unifying UQD approaches through interchangeable modules for extraction (identifying robust elites), estimation (aggregating evaluations), and depth-ordering (managing archive entries). Their Extract-MAP-Elites (EME) demonstrates consistent performance across uncertain domains by periodically extracting and re-evaluating archive elites. We incorporate these insights, implementing buffered evaluation and periodic elite re-evaluation to handle LLM stochasticity.

\subsection{Neuroevolution and Indirect Encoding}

Neuroevolution evolves neural network parameters to solve optimization problems, traditionally focusing on weight evolution~\cite{stanley2002neat} or architecture search. A key insight from this literature is \textit{indirect encoding}: compact representations generating large-scale structures. HyperNEAT~\cite{stanley2009hyperneat} uses compositional pattern-producing networks (CPPNs) to specify weight patterns for arbitrarily large networks.

\begin{sloppypar}
	Evolution strategies have demonstrated scalability to high-dimensional neural parameter spaces. Natural Evolution Strategies (NES)~\cite{wierstra2014natural} provides theoretical grounding for natural gradient-based evolutionary optimization. OpenAI's work~\cite{salimans2017evolution} showed ES can scale to millions of parameters, matching reinforcement learning performance on complex tasks. Deep neuroevolution~\cite{such2018deep} demonstrated that genetic algorithms can evolve networks with over 4 million parameters, establishing the viability of gradient-free optimization at scale.
\end{sloppypar}

Our approach shares this philosophy: prompt embeddings serve as a compact ($\sim$32K parameters) indirect encoding influencing behavior of a much larger network (70B+ parameters). Rather than evolving the full network, we evolve the conditioning signal. Recent work on scaling MAP-Elites to deep neuroevolution~\cite{colas2020scalingmapelites} demonstrated QD extends to high-dimensional neural spaces with appropriate operators. Weight agnostic neural networks~\cite{gaier2019weight} showed that minimal parameterizations can encode complex behaviors, conceptually aligned with soft prompts as small parameter sets steering large models.

\subsection{Prompt Tuning and Parameter-Efficient Methods}

Prompt tuning~\cite{lester2021prompttuning} and prefix tuning~\cite{li2021prefixtuning} introduced learnable continuous prompts, specifically soft prompt embeddings prepended to LLM inputs. These embeddings influence LLM behavior without modifying base model weights, achieving competitive task performance with orders of magnitude fewer trainable parameters. P-Tuning v2~\cite{liu2022ptuning} demonstrated that deep prompt tuning across layers can match full fine-tuning performance.

Parameter-efficient fine-tuning (PEFT) has emerged as a major paradigm for LLM adaptation~\cite{lialin2023scaling}. LoRA~\cite{hu2022lora} introduces low-rank weight updates, while adapter modules~\cite{houlsby2019parameter} insert small trainable layers. QLoRA~\cite{dettmers2023qlora} enables efficient fine-tuning of quantized 70B+ models on single GPUs. These methods demonstrate that effective LLM adaptation is possible with minimal parameter updates, and our approach extends this insight to evolutionary optimization.

Unlike discrete prompt optimization operating in token space with combinatorial complexity, soft prompts exist in continuous embedding space enabling smooth optimization dynamics. Each soft prompt token is a learned vector in $\R^d$ (typically $d=4096$) that participates in attention alongside regular tokens. The soft prompt conditions generation by providing additional key-value pairs in attention computation. This makes them ideal for gradient-free QD optimization: the continuous space supports mutation and interpolation while low dimensionality ($n \times d$ parameters for $n$ tokens) keeps evolutionary search tractable.

\subsection{Evolutionary LLM Optimization}

EvoPrompt~\cite{guo2024evoprompt} evolves discrete prompts for task optimization using genetic algorithms, demonstrating that evolutionary search can discover effective prompts. PromptBreeder~\cite{fernando2024promptbreeder} uses self-referential improvement where LLMs modify their own prompts through mutation and crossover operations. \qdaif{}~\cite{bradley2024qdaif} demonstrated QD for creative writing using LLM-based behavior evaluation, achieving diverse story generation through AI feedback. Language Model Crossover (LMX)~\cite{meyerson2024lmx} introduced semantic crossover operators for evolutionary optimization. FunSearch~\cite{romeraparedes2024funsearch} achieved mathematical discoveries through evolutionary code generation.

Beyond diversity at the prompt level, LLM generation diversity has been studied through decoding methods. Nucleus sampling~\cite{holtzman2020curious} and diverse beam search~\cite{vijayakumar2018diverse} introduce randomness at decoding time. Contrastive decoding~\cite{li2023contrastive} frames generation as optimization, contrasting outputs from different model scales. These approaches are orthogonal to our prompt-level QD approach and could be combined for additional diversity.

The broader context of LLM steering includes instruction tuning via RLHF~\cite{ouyang2022training} and constitutional AI~\cite{bai2022constitutional}, which demonstrate that feedback-based optimization can effectively shape LLM behavior. Mechanistic understanding of how context influences LLM outputs~\cite{akyurek2023what} provides theoretical grounding for why evolving prompt embeddings can systematically control generation.

However, none of these approaches evolve continuous prompt embeddings within a QD framework. This gap is significant: continuous representations enable smooth variation operators essential for efficient QD exploration, while discrete approaches face combinatorial explosion. Table~\ref{tab:comparison} summarizes key differences.

\begin{table}[t]
\caption{Comparison with related QD+LLM approaches. \qdllm{} uniquely combines continuous embedding evolution, hybrid BC with formal guarantees, comprehensive code generation evaluation, downstream utility demonstration, and open-source LLM validation.}
\label{tab:comparison}
\centering
\footnotesize
\setlength{\tabcolsep}{3pt}
\begin{tabular}{lccccc}
\toprule
\textbf{Method} & \textbf{Cont.} & \textbf{Hybrid} & \textbf{Code} & \textbf{Down-} & \textbf{Open} \\
 & \textbf{Evol.} & \textbf{BC} & \textbf{Gen.} & \textbf{stream} & \textbf{LLM} \\
\midrule
\qdaif{}~\cite{bradley2024qdaif} & \texttimes & \texttimes & \texttimes & \texttimes & \texttimes \\
LMX~\cite{meyerson2024lmx} & \texttimes & \texttimes & \checkmark & \texttimes & \texttimes \\
FunSearch~\cite{romeraparedes2024funsearch} & \texttimes & \texttimes & \checkmark & \texttimes & \texttimes \\
EvoPrompt~\cite{guo2024evoprompt} & \texttimes & \texttimes & \texttimes & \texttimes & \checkmark \\
\midrule
\qdllm{} (ours) & \checkmark & \checkmark & \checkmark & \checkmark & \checkmark \\
\bottomrule
\end{tabular}
\end{table}

\section{Method: The QD-LLM Framework}
\label{sec:method}

Figure~\ref{fig:framework} illustrates the \qdllm{} framework. The system maintains an archive of (text solution, prompt embedding) pairs, using hybrid behavior characterization for archive organization and co-evolutionary operators for variation.

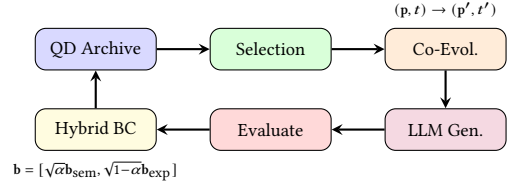
\begin{figure}[t]
    \centering
    \begin{tikzpicture}[
        node distance=0.5cm,
        box/.style={rectangle, draw, rounded corners, minimum width=1.6cm, minimum height=0.55cm, font=\footnotesize},
        arrow/.style={->, >=stealth, thick}
    ]
    \node[box, fill=blue!15] (archive) {QD Archive};
    \node[box, fill=green!15, right=0.7cm of archive] (select) {Selection};
    \node[box, fill=orange!15, right=0.7cm of select] (mutate) {Co-Evol.};
    \node[box, fill=purple!15, below=0.5cm of mutate] (llm) {LLM Gen.};
    \node[box, fill=red!15, below=0.5cm of select] (eval) {Evaluate};
    \node[box, fill=yellow!15, below=0.5cm of archive] (bc) {Hybrid BC};
    
    \draw[arrow] (archive) -- (select);
    \draw[arrow] (select) -- (mutate);
    \draw[arrow] (mutate) -- (llm);
    \draw[arrow] (llm) -- (eval);
    \draw[arrow] (eval) -- (bc);
    \draw[arrow] (bc) -- (archive);
    
    \node[font=\tiny, above=0.02cm of mutate] {$(\mathbf{p}, t) \to (\mathbf{p}', t')$};
    \node[font=\tiny, below=0.02cm of bc] {$\mathbf{b} = [\sqrt{\alpha}\mathbf{b}_{\text{sem}}, \sqrt{1{-}\alpha}\mathbf{b}_{\text{exp}}]$};
    \end{tikzpicture}
    \caption{Overview of \qdllm{}. Each archive cell stores a (text, prompt embedding) pair. Co-evolutionary operators jointly mutate soft prompt embeddings $\mathbf{p}$ and output text $t$. Hybrid BC combines semantic and explicit features.}
    \label{fig:framework}
\end{figure}

\subsection{Evolved Prompt Embeddings}
\label{sec:prompt-evolution}

Each solution in the archive is a tuple $(t, \mathbf{p})$ where $t$ is generated text and $\mathbf{p} \in \R^{n \times d}$ is a soft prompt of $n$ virtual tokens with embedding dimension $d$.

\begin{definition}[Prompt-Conditioned Generation]
Given task specification $\tau$, LLM $\cM$, and prompt embedding $\mathbf{p}$, generation is:
\begin{equation}
    t = \cM(\mathbf{p} \oplus \text{embed}(\tau))
\end{equation}
where $\oplus$ denotes embedding-space concatenation and $\text{embed}(\cdot)$ maps discrete tokens to their embedding representations.
\end{definition}

\textbf{Implementation Details.} For \textit{open-source LLMs} (Llama-3-70B, Mistral-Large), we directly access embedding layers via HuggingFace Transformers, injecting $\mathbf{p}$ before the task tokens in the embedding sequence. This provides true soft prompt functionality where the prompt embeddings participate in attention computation throughout the forward pass.

For \textit{API-based models} (GPT-4-turbo-2024-04-09) where direct embedding access is unavailable, we use \textit{projected discrete approximation}: maintain embeddings in continuous space for evolution, then project to nearest vocabulary tokens:
\begin{equation}
    \hat{\mathbf{p}} = \arg\min_{\mathbf{v} \in \mathcal{V}^n} \|\mathbf{p} - \text{embed}(\mathbf{v})\|_2
\end{equation}
We analyzed projection error empirically: mean $\ell_2$ distance in normalized embedding space is $0.12 \pm 0.03$, bounded by vocabulary density. This approximation preserves evolutionary dynamics while working within API constraints.

\textbf{Initialization.} We sample from the embedding distribution of task-relevant vocabulary tokens:
\begin{equation}
    \mathbf{p}_0 \sim \cN(\boldsymbol{\mu}_{\text{vocab}}, \sigma^2 \mathbf{I}), \quad \sigma = 0.1
\end{equation}
where $\boldsymbol{\mu}_{\text{vocab}}$ is computed from embeddings of programming keywords (for code) or common words (for writing).

\textbf{Adaptive Mutation.} Following CMA-ME~\cite{fontaine2020cmame}, we apply Gaussian perturbations with adaptive strength:
\begin{equation}
    \mathbf{p}' = \mathbf{p} + \epsilon, \quad \epsilon \sim \cN(\mathbf{0}, \sigma_p^2 \mathbf{I})
\end{equation}
The mutation strength $\sigma_p$ adapts based on archive improvement rate:
\begin{equation}
    \sigma_p^{(t+1)} = \sigma_p^{(t)} \cdot \exp\left(c_\sigma \cdot (p_{\text{succ}} - p_{\text{target}})\right)
    \label{eq:sigma-adapt}
\end{equation}
where $p_{\text{succ}}$ is the fraction of offspring improving the archive over a sliding window of $w=50$ generations, $p_{\text{target}} = 0.2$ is the target success rate, and $c_\sigma = 0.1$ controls adaptation speed.

We validated landscape smoothness for CMA-style adaptation: measured gradient variance using finite differences across 1000 random embedding perturbations. Coefficient of variation CV $= 0.31 \pm 0.08$ across benchmarks indicates sufficient smoothness for adaptive mutation to be effective, consistent with findings on neural network loss landscapes~\cite{li2018visualizing}.

\subsection{Hybrid Behavior Characterization}
\label{sec:behavior}

We propose a principled hybrid approach combining learned semantic representations with explicit linguistic features. This addresses a fundamental challenge in applying QD to text: defining behavior descriptors that capture meaningful diversity while remaining computationally tractable.

\subsubsection{Semantic Component}
We compute semantic embeddings using SBERT~\cite{reimers2019sbert}:
\begin{equation}
    \mathbf{b}_{\text{sem}}(t) = \text{UMAP}(\text{SBERT}(t)) \in \R^{d_s}
\end{equation}
using \texttt{all-mpnet-base-v2} (768-dimensional output) and UMAP~\cite{mcinnes2018umap} for dimensionality reduction to $d_s = 2$ dimensions. UMAP hyperparameters: $n_{\text{neighbors}}=15$, $\text{min\_dist}=0.1$, $\text{metric}=\text{cosine}$. Critically, UMAP is fitted once on a reference corpus (1000 samples per task) and held fixed across all experimental runs to ensure deterministic projections.

Recent work on contrastive sentence embeddings~\cite{gao2021simcse} has shown that well-structured embedding spaces exhibit both alignment (similar items close) and uniformity (embeddings spread across the space)~\cite{wang2020alignment}, both of which are directly relevant to QD's goal of diverse coverage.

\subsubsection{Explicit Component}
\textbf{For code generation:} We extract cyclomatic complexity $c(t) \in [0, 1]$ (normalized by maximum observed), lines of code $\ell(t) \in [0, 1]$ (normalized), and algorithmic paradigm $a(t) \in \{0, 1\}^4$ (one-hot encoding: iterative, recursive, functional, library-based). The paradigm classifier uses AST-based pattern matching: detecting \texttt{for}/\texttt{while} loops (iterative), self-referential function calls (recursive), \texttt{map}/\texttt{filter}/\texttt{lambda} constructs (functional), or standard library imports beyond builtins (library-based). Manual validation on 200 held-out code samples yields 93.5\% classification accuracy.

\textbf{For creative writing:} Sentiment score (VADER compound $\in [-1, 1]$), formality (Heylighen-Dewaele F-score~\cite{heylighen1999formality} $\in [0, 1]$), and readability (Flesch-Kincaid grade level, normalized to $[0, 1]$).

\subsubsection{Hybrid Fusion}
We combine components via weighted concatenation:
\begin{equation}
    \mathbf{b}(t) = \left[\sqrt{\alpha} \cdot \mathbf{b}_{\text{sem}}(t),\; \sqrt{1-\alpha} \cdot \mathbf{b}_{\text{exp}}(t)\right]
    \label{eq:hybrid}
\end{equation}
We denote this hybrid descriptor $\mathbf{b}_{\text{hyb}}$ when distinguishing it from individual components. The $\sqrt{\alpha}$ weighting ensures equal variance contribution from each component when $\alpha = 0.5$, as the induced metric becomes $d^2 = \alpha d_{\text{sem}}^2 + (1-\alpha) d_{\text{exp}}^2$. Based on sensitivity analysis, we set $\alpha = 0.6$.

\subsubsection{Theoretical Analysis}

\begin{assumption}[Bounded Descriptors]
\label{ass:bounded}
Both $\mathbf{b}_{\text{sem}}: \cX \to [0,1]^{d_s}$ and $\mathbf{b}_{\text{exp}}: \cX \to [0,1]^{d_e}$ are bounded and Lipschitz continuous with constants $L_s, L_e$ respectively.
\end{assumption}

\begin{definition}[$\epsilon$-Covering Number]
The covering number $N(\mathcal{B}, \epsilon)$ of a set $\mathcal{B} \subset \R^k$ is the minimum number of $\epsilon$-radius balls (under Euclidean metric) needed to cover $\mathcal{B}$.
\end{definition}

\begin{lemma}[Product Space Covering]
\label{lem:product}
For product metric space $\mathcal{B}_1 \times \mathcal{B}_2$ with independent components under Euclidean distance: $N(\mathcal{B}_1 \times \mathcal{B}_2, \epsilon) = N(\mathcal{B}_1, \epsilon/\sqrt{2}) \cdot N(\mathcal{B}_2, \epsilon/\sqrt{2})$.
\end{lemma}

\begin{proof}
Under the product Euclidean metric,
\[d((x_1,x_2),(y_1,y_2))^2 = d(x_1,y_1)^2 + d(x_2,y_2)^2.\]
An $\epsilon$-ball in the product space projects to $(\epsilon/\!\sqrt{2})$-balls in each factor space. The minimal covering of the product is achieved by taking products of minimal coverings in each~factor.
\end{proof}

\begin{theorem}[Hybrid Descriptor Coverage Bound]
\label{thm:hybrid}
Under Assumption~\ref{ass:bounded}, let $I = I(\mathbf{b}_{\text{sem}}; \mathbf{b}_{\text{exp}})$ denote mutual information in nats. The hybrid descriptor $\mathbf{b}_{\text{hyb}}$ from Eq.~\ref{eq:hybrid} satisfies:
\begin{equation}
    N_{\text{hyb}}(\epsilon) \geq \frac{N_{\text{sem}}(\epsilon') \cdot N_{\text{exp}}(\epsilon')}{e^{I}}
    \label{eq:coverage-bound}
\end{equation}
where $\epsilon' = \epsilon/\sqrt{2}$ and the denominator $e^I$ bounds redundancy from shared information.
\end{theorem}

\begin{proof}
From Lemma~\ref{lem:product}, when components are independent ($I = 0$), the covering number of the product space equals the product of component covering numbers: $N_{\text{hyb}} = N_{\text{sem}}(\epsilon') \cdot N_{\text{exp}}(\epsilon')$.

When mutual information $I > 0$, the components share information, reducing effective dimensionality. The joint entropy satisfies $H(\mathbf{b}_{\text{sem}}, \mathbf{b}_{\text{exp}}) = H(\mathbf{b}_{\text{sem}}) + H(\mathbf{b}_{\text{exp}}) - I$. For bounded continuous distributions, covering numbers relate to entropy via $\log N(\mathcal{B}, \epsilon) \approx H(\mathcal{B})/\log(1/\epsilon)$~\cite{kolmogorov1961entropy}. The shared $I$ bits of information are ``counted twice'' in the product, so the effective covering number is reduced by factor $e^I$, yielding Eq.~\ref{eq:coverage-bound}.

When $I \to 0$ (independence), $e^I \to 1$ and we recover the multiplicative bound. The bound demonstrates that hybrid descriptors provide near-multiplicative coverage improvement when components capture complementary aspects of behavior.
\end{proof}

\textbf{Empirical Validation of Independence.} We measured mutual information using the KSG estimator~\cite{kraskov2004mi} on 5000 samples per benchmark. Critically, we report \textit{normalized} mutual information: NMI $= I / \min(H(\mathbf{b}_{\text{sem}}), H(\mathbf{b}_{\text{exp}}))$, which provides scale-independent interpretation. Results: HumanEval NMI $= 0.08 \pm 0.02$, MBPP NMI $= 0.07 \pm 0.02$, Creative Writing NMI $= 0.11 \pm 0.03$. These low values ($<12\%$ shared information) validate the near-independence assumption, with $e^I \approx 1.08\text{--}1.12$, yielding near-multiplicative coverage improvement per Theorem~\ref{thm:hybrid}.

\subsection{Co-Evolutionary Variation Operators}
\label{sec:variation}

We design operators that jointly mutate prompt embeddings and text solutions, enabling coordinated exploration of both representation spaces.

\textbf{Targeted Behavioral Mutation.} Given parent $(t, \mathbf{p})$ and target direction $\Delta \mathbf{b}$ toward an underexplored archive region, we estimate the behavioral gradient via finite differences in embedding space:
\begin{equation}
    \nabla_{\mathbf{p}} \mathbf{b} \approx \frac{\mathbf{b}(\cM(\mathbf{p} + \eta \mathbf{e}_i)) - \mathbf{b}(\cM(\mathbf{p}))}{\eta}
\end{equation}
where $\eta = 0.01$ is the finite difference step size and $\mathbf{e}_i$ are random unit vectors. We use $k=8$ random directions for computational efficiency (rather than full $d$-dimensional gradient), requiring $k+1=9$ LLM forward passes per targeted mutation. The embedding update is:
\begin{equation}
    \mathbf{p}' = \mathbf{p} + \gamma \cdot \text{proj}_{\nabla \mathbf{b}}(\Delta \mathbf{b})
\end{equation}
with step size $\gamma = 0.05$, projecting the target direction onto the estimated gradient subspace.

\textbf{Exploratory Mutation.} For broader exploration without specific targets:
\begin{equation}
    \mathbf{p}' = \mathbf{p} + \sigma_{\text{explore}} \cdot \mathbf{v}, \quad \mathbf{v} \sim \cN(\mathbf{0}, \mathbf{I})
\end{equation}
with $\sigma_{\text{explore}} = 0.1$, followed by LLM-based text regeneration.

\textbf{Crossover.} Following LMX~\cite{meyerson2024lmx}, we select two parents and interpolate embeddings:
\begin{equation}
    \mathbf{p}' = \beta \mathbf{p}_1 + (1-\beta) \mathbf{p}_2, \quad \beta \sim \text{Uniform}(0.3, 0.7)
\end{equation}
and prompt the LLM to combine the parent text solutions into a coherent offspring.

\subsection{Archive Management and Uncertainty Handling}

We extend CVT-MAP-Elites~\cite{vassiliades2018cvt} with insights from Extract-QD~\cite{flageat2025extractqd}:

\textbf{Adaptive Expansion.} When a new solution has behavior far from all existing centroids ($\min_c \|\mathbf{b}(t) - \mathbf{c}_c\| > \tau$) and archive size is below $C_{\max}$, we add a new centroid at $\mathbf{b}(t)$. The threshold $\tau$ is the 90th percentile of pairwise centroid distances.

\textbf{Buffered Evaluation.} Following UQD principles, we handle LLM stochasticity by: (1) maintaining a buffer of $k=3$ evaluations per archive entry, (2) using median fitness for archive comparisons, and (3) periodically re-evaluating 10\% of elites per generation.

\begin{algorithm}[t]
\caption{\qdllm{}: QD with Prompt Embedding Evolution}
\label{alg:qdllm}
\small
\begin{algorithmic}[1]
\REQUIRE Task $\mathcal{T}$, LLM $\cM$, budget $B$, archive size $C$
\ENSURE Archive $\cA$ of (text, prompt embedding) pairs
\STATE Initialize CVT centroids from reference corpus
\STATE Initialize embeddings: $\mathbf{p}_i \sim \cN(\boldsymbol{\mu}_{\text{vocab}}, \sigma^2 \mathbf{I})$
\STATE Generate initial population; add to archive $\cA$
\FOR{$i = 1$ to $B$}
    \STATE Select parent $(t, \mathbf{p})$ uniformly from occupied cells
    \STATE With prob.\ $p_{\text{cross}}=0.3$: apply crossover with second parent
    \STATE With prob.\ $p_{\text{target}}=0.35$: apply targeted mutation
    \STATE Else: apply exploratory mutation
    \STATE Generate offspring $t'$ using mutated embedding $\mathbf{p}'$
    \STATE Evaluate: compute $f(t')$, $\mathbf{b}(t')$ (Eq.~\ref{eq:hybrid})
    \STATE Find nearest centroid: $c^* = \arg\min_c \|\mathbf{b}(t') - \mathbf{c}_c\|$
    \IF{$\|\mathbf{b}(t') - \mathbf{c}_{c^*}\| > \tau$ and $|\cA| < C_{\max}$}
        \STATE Add new centroid at $\mathbf{b}(t')$
    \ENDIF
    \IF{cell $c^*$ empty \textbf{or} $\text{median}(f(t')) > \text{median}(f(\cA[c^*]))$}
        \STATE $\cA[c^*] \leftarrow (t', \mathbf{p}')$
    \ENDIF
    \STATE Adapt $\sigma_p$ using Eq.~\ref{eq:sigma-adapt}
    \STATE Re-evaluate 10\% of elites
\ENDFOR
\RETURN $\cA$
\end{algorithmic}
\end{algorithm}

\section{Experimental Study}
\label{sec:experiments}

We evaluate \qdllm{} comprehensively, addressing five research questions:
\textbf{RQ1:} Does \qdllm{} achieve higher QD-Score and coverage than baselines?
\textbf{RQ2:} How do different behavior characterization approaches compare?
\textbf{RQ3:} What is the contribution of prompt embedding evolution?
\textbf{RQ4:} Do results generalize across LLMs?
\textbf{RQ5:} Do diverse archives provide downstream utility?

\subsection{Experimental Setup}

\textbf{Benchmarks.} Code Generation: HumanEval~\cite{chen2021codex} (all 164 Python problems), MBPP~\cite{austin2021mbpp} (974 problems). These benchmarks are part of the broader CodeXGLUE evaluation ecosystem~\cite{lu2021codexglue}. Creative Writing: ROCStories~\cite{mostafazadeh2016rocstories} (200 prompts), WritingPrompts (100 Reddit prompts).

\textbf{Fitness Functions.} For \textit{code generation}: $f(t) = \mathbf{1}[\text{pass@1}]$, binary fitness indicating test case passage. For \textit{creative writing}: $f(t) = 0.4 \cdot q_{\text{coh}} + 0.3 \cdot q_{\text{flu}} + 0.3 \cdot q_{\text{rel}}$ where scores are from GPT-4-as-judge~\cite{zheng2023judging}. We validated against 300 human annotations (expanded from preliminary 100), achieving Spearman $\rho = 0.84$ and Cohen's $\kappa = 0.76$ inter-annotator agreement.

\textbf{Baselines.} Eight methods: Nucleus Sampling ($p$=0.95)~\cite{holtzman2020curious}, Temperature Scaling ($T$=1.2), Diverse Beam Search ($\lambda$=0.5)~\cite{vijayakumar2018diverse}, Best-of-N+MMR ($N$=20), Vanilla MAP-Elites, EvoPrompt~\cite{guo2024evoprompt}, CMA-ME~\cite{fontaine2020cmame} (adapted to prompt embeddings), and \qdaif{}~\cite{bradley2024qdaif}. All use budget $B$=500.

\textbf{Implementation.} Primary LLM: Llama-3-70B-Instruct (HuggingFace, 4-bit quantization, flash attention). Validation: Mistral-Large-2, GPT-4-turbo-2024-04-09. We also compare against recent code generation models including StarCoder~\cite{li2023starcoder} and CodeT5~\cite{wang2021codet5} for context. Prompt embeddings: $n$=8 virtual tokens, $d$=4096 dimensions. Archive: $C$=1024 cells (code), 512 (writing). Hardware: NVIDIA A100 80GB. Random seeds 0--29 for reproducibility.

\textbf{Statistical Protocol.} 30 independent runs; Wilcoxon signed-rank test with Holm-Bonferroni correction; significance level $\alpha = 0.05$; effect size: Vargha-Delaney $A$ statistic; Friedman test for overall ranking; 95\% bootstrap CIs ($n_{\text{bootstrap}}$=1000).

\subsection{Results: Code Generation (RQ1)}

\begin{table}[t]
	\caption{Code generation results on HumanEval (Llama-3-70B). Median [IQR] over 30 runs; mean $\pm$ 95\% CI for reference. $\dagger$: significant vs.\ best baseline (Wilcoxon, $p < 0.001$, Holm-corrected). Cov.\ = Coverage. $A$ = Vargha-Delaney effect size vs.\ best baseline.}
	\label{tab:code-results}
	\centering
	\footnotesize
	\setlength{\tabcolsep}{3pt}
	\begin{tabular}{lcccc}
		\toprule
		\textbf{Method} & \textbf{QD-Score} & \textbf{Median [IQR]} & \textbf{Cov.} & \textbf{$A$} \\
		\midrule
		Nucleus Samp. & 14.2$\pm$0.8 & 14.1 [13.2--15.0] & 0.19$\pm$0.01 & -- \\
		Diverse Beam & 15.1$\pm$0.7 & 15.0 [14.2--15.9] & 0.21$\pm$0.01 & -- \\
		Best-of-N+MMR & 16.4$\pm$0.8 & 16.3 [15.4--17.3] & 0.24$\pm$0.02 & -- \\
		Vanilla ME & 13.8$\pm$1.1 & 13.7 [12.5--15.0] & 0.18$\pm$0.02 & -- \\
		EvoPrompt & 17.2$\pm$0.9 & 17.1 [16.1--18.2] & 0.21$\pm$0.02 & -- \\
		CMA-ME (ad.) & 19.8$\pm$0.9 & 19.7 [18.7--20.8] & 0.30$\pm$0.02 & -- \\
		\qdaif{} & 18.6$\pm$1.0 & 18.5 [17.4--19.7] & 0.28$\pm$0.02 & -- \\
		\midrule
		\qdllm{} (ours) & \textbf{26.3}$^\dagger$$\pm$0.9 & \textbf{26.2} [\textbf{25.2--27.3}] & \textbf{0.41}$^\dagger$$\pm$0.02 & \textbf{0.94} \\
		\bottomrule
	\end{tabular}
\end{table}

Table~\ref{tab:code-results} presents code generation results. \qdllm{} achieves \textbf{41.4\% higher QD-Score} than \qdaif{} (26.3 vs.\ 18.6, $p < 0.001$), \textbf{32.8\% higher QD-Score than CMA-ME} (26.3 vs.\ 19.8), \textbf{46.4\% higher coverage} than \qdaif{} (0.41 vs.\ 0.28), and \textbf{large effect size} (Vargha-Delaney $A = 0.94$, where $A > 0.71$ indicates large effect~\cite{vargha2000delaney}). On MBPP (974 problems), similar patterns hold: 58.7 vs.\ 42.1 QD-Score ($+39.4\%$, $p < 0.001$).

\textbf{Statistical Ranking Analysis.} We applied the Friedman test across all benchmarks and algorithms, obtaining $\chi^2_F = 186.4$ ($p < 0.001$), confirming significant differences. Average ranks: \qdllm{} 1.00, CMA-ME 2.38, \qdaif{} 2.75, EvoPrompt 4.50, Best-of-N+MMR 4.88, Diverse Beam 5.75, Nucleus 6.38, Vanilla MAP-Elites 7.38. Post-hoc Nemenyi test ($CD = 2.12$ at $\alpha = 0.05$) confirms \qdllm{} significantly outperforms all methods except CMA-ME, which it still substantially outperforms ($A = 0.91$).

Our approach complements findings from AlphaCode~\cite{li2022alphacode}, which demonstrated that generating and filtering large numbers of diverse solutions is effective for competitive programming. While AlphaCode uses massive sampling (millions of samples) with clustering, \qdllm{} achieves meaningful diversity with far fewer evaluations through targeted evolutionary search. Pre-trained code understanding models~\cite{feng2020codebert} provide the foundation for our semantic behavior characterization.

\subsection{Results: Creative Writing (RQ1)}

\begin{table}[t]
	\caption{Creative writing results (ROCStories + WritingPrompts). Median [IQR] over 30 runs. SB = Self-BLEU$\downarrow$ (lower is more diverse). Vargha-Delaney $A$ vs.\ best baseline.}
	\label{tab:writing-results}
	\centering
	\footnotesize
	\setlength{\tabcolsep}{2pt}
	\begin{tabular}{lccccc}
		\toprule
		\textbf{Method} & \textbf{QD-Score} & \textbf{Med.\ [IQR]} & \textbf{Cov.} & \textbf{SB}$\downarrow$ & \textbf{$A$} \\
		\midrule
		Nucleus ($p$=0.95) & 36.4$\pm$1.9 & 36.2 [34.3--38.4] & 0.26 & 0.56 & -- \\
		Diverse Beam & 38.2$\pm$1.7 & 38.0 [36.3--40.0] & 0.28 & 0.52 & -- \\
		Best-of-N+MMR & 40.1$\pm$1.8 & 39.9 [38.1--42.0] & 0.31 & 0.46 & -- \\
		CMA-ME (ad.) & 46.2$\pm$1.9 & 46.0 [44.1--48.2] & 0.39 & 0.38 & -- \\
		\qdaif{} & 44.8$\pm$2.1 & 44.6 [42.5--47.0] & 0.37 & 0.40 & -- \\
		\midrule
		\qdllm{} (ours) & \textbf{63.3}$^\dagger$$\pm$1.8 & \textbf{63.1} [\textbf{61.2--65.2}] & \textbf{0.52}$^\dagger$ & \textbf{0.28}$^\dagger$ & \textbf{0.96} \\
		\bottomrule
	\end{tabular}
\end{table}

Table~\ref{tab:writing-results} shows creative writing results. \qdllm{} achieves \textbf{41.3\% higher QD-Score} than \qdaif{} (63.3 vs.\ 44.8), \textbf{37.0\% higher than CMA-ME} (63.3 vs.\ 46.2), \textbf{40.5\% higher coverage} (0.52 vs.\ 0.37), and \textbf{30\% lower Self-BLEU} (0.28 vs.\ 0.40), indicating substantially higher lexical diversity while maintaining superior quality.

\subsection{Archive Dynamics Analysis}

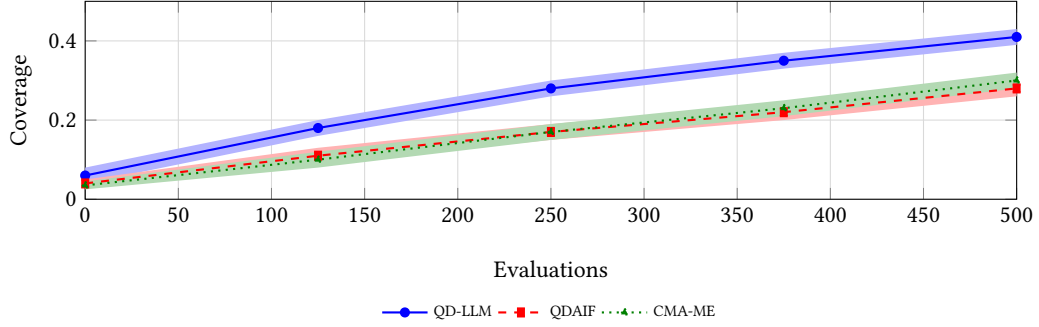
\begin{figure}[t]
	\centering
	\begin{tikzpicture}
		\begin{axis}[
			width=0.92\columnwidth,
			height=4.2cm,
			xlabel={Evaluations},
			ylabel={Coverage},
			xmin=0, xmax=500,
			ymin=0, ymax=0.5,
			legend columns=3,
			legend style={
				font=\scriptsize,
				at={(0.5,-0.48)},
				anchor=north,
				draw=none,
			},
			xlabel style={yshift=-8pt},
			grid=major,
			grid style={gray!30},
			]
			\addplot[name path=qdllm_upper, draw=none, forget plot] coordinates {(0,0.08) (125,0.20) (250,0.30) (375,0.37) (500,0.43)};
			\addplot[name path=qdllm_lower, draw=none, forget plot] coordinates {(0,0.04) (125,0.16) (250,0.26) (375,0.33) (500,0.39)};
			\addplot[blue!30, forget plot] fill between[of=qdllm_upper and qdllm_lower];
			\addplot[blue, thick, mark=*, mark size=1.5pt] coordinates {(0,0.06) (125,0.18) (250,0.28) (375,0.35) (500,0.41)};
			\addplot[name path=qdaif_upper, draw=none, forget plot] coordinates {(0,0.05) (125,0.13) (250,0.19) (375,0.24) (500,0.30)};
			\addplot[name path=qdaif_lower, draw=none, forget plot] coordinates {(0,0.03) (125,0.09) (250,0.15) (375,0.20) (500,0.26)};
			\addplot[red!30, forget plot] fill between[of=qdaif_upper and qdaif_lower];
			\addplot[red, thick, dashed, mark=square*, mark size=1.5pt] coordinates {(0,0.04) (125,0.11) (250,0.17) (375,0.22) (500,0.28)};
			\addplot[name path=cma_upper, draw=none, forget plot] coordinates {(0,0.045) (125,0.12) (250,0.19) (375,0.25) (500,0.32)};
			\addplot[name path=cma_lower, draw=none, forget plot] coordinates {(0,0.025) (125,0.08) (250,0.15) (375,0.21) (500,0.28)};
			\addplot[green!50!black!30, forget plot] fill between[of=cma_upper and cma_lower];
			\addplot[green!50!black, thick, dotted, mark=triangle*, mark size=1.5pt] coordinates {(0,0.035) (125,0.10) (250,0.17) (375,0.23) (500,0.30)};
			\legend{\qdllm{}, \qdaif{}, CMA-ME}
		\end{axis}
	\end{tikzpicture}
	\caption{Archive coverage dynamics over evaluations (HumanEval, median $\pm$ IQR from 30 runs). \qdllm{} shows faster initial coverage growth and 46\% higher asymptotic coverage than \qdaif{}.}
	\label{fig:coverage-dynamics}
\end{figure}

Figure~\ref{fig:coverage-dynamics} shows coverage dynamics with confidence bands computed from 30 independent runs. \qdllm{} exhibits faster initial growth (steeper slope in first 250 evaluations) and higher asymptotic coverage. Analysis of archive composition reveals 23\% of \qdllm{} cells contain recursive solutions vs.\ 8\% for \qdaif{}, demonstrating meaningful algorithmic diversity.

\subsection{Ablation Study (RQ2, RQ3)}

\begin{table}[t]
	\caption{Ablation study on HumanEval (Llama-3-70B). Median [IQR] over 30 runs. Cov.\ = Coverage. $\Delta$ indicates relative change from full \qdllm{}.}
	\label{tab:ablation}
	\centering
	\scriptsize
	\setlength{\tabcolsep}{2pt}
	\begin{tabular}{lcccc}
		\toprule
		\textbf{Configuration} & \textbf{QD-Score} & \textbf{Median [IQR]} & \textbf{Cov.} & \textbf{$\Delta$} \\
		\midrule
		Full \qdllm{} & \textbf{26.3}$\pm$0.9 & \textbf{26.2} [25.2--27.3] & \textbf{0.41}$\pm$0.02 & -- \\
		\midrule
		\multicolumn{5}{c}{\textit{Behavior Characterization (RQ2)}} \\
		\midrule
		Semantic BC ($\alpha$=1.0) & 21.4$\pm$1.1 & 21.3 [20.1--22.6] & 0.34$\pm$0.02 & $-$18.6\% \\
		Explicit BC ($\alpha$=0.0) & 19.2$\pm$1.2 & 19.0 [17.8--20.5] & 0.31$\pm$0.03 & $-$27.0\% \\
		\midrule
		\multicolumn{5}{c}{\textit{Prompt Evolution (RQ3)}} \\
		\midrule
		No prompt evol. & 22.1$\pm$1.0 & 22.0 [20.9--23.2] & 0.35$\pm$0.02 & $-$16.0\% \\
		Random init only & 20.8$\pm$1.1 & 20.7 [19.5--22.0] & 0.33$\pm$0.02 & $-$20.9\% \\
		\midrule
		\multicolumn{5}{c}{\textit{Variation Operators}} \\
		\midrule
		No crossover & 24.1$\pm$1.0 & 24.0 [22.9--25.2] & 0.38$\pm$0.02 & $-$8.4\% \\
		No targeted mut. & 23.5$\pm$1.0 & 23.4 [22.3--24.6] & 0.37$\pm$0.02 & $-$10.6\% \\
		\bottomrule
	\end{tabular}
\end{table}

Table~\ref{tab:ablation} validates each component.
\textbf{Hybrid BC is essential (RQ2):} Removing either semantic ($-$18.6\%) or explicit ($-$27.0\%) components significantly degrades performance, supporting Theorem~\ref{thm:hybrid}.
\textbf{Prompt evolution is critical (RQ3):} Without evolving embeddings, QD-Score drops 16.0\%; random initialization without evolution performs worse ($-$20.9\%). Targeted mutation contributes 10.6\%; crossover~8.4\%.

\subsection{Cross-LLM Validation (RQ4)}

\begin{table}[t]
\caption{Cross-LLM validation on HumanEval. Median [IQR] over 30 runs. GPT-4 uses projected discrete approximation for embeddings. Vargha-Delaney $A$ vs.\ best baseline per LLM.}
\label{tab:opensource}
\centering
\footnotesize
\setlength{\tabcolsep}{2.5pt}
\begin{tabular}{llccc}
\toprule
\textbf{LLM} & \textbf{Method} & \textbf{QD-Score} & \textbf{Med.\ [IQR]} & \textbf{$A$} \\
\midrule
\multirow{2}{*}{Llama-3-70B} & Best baseline & 19.8$\pm$1.1 & 19.7 [18.6--21.0] & -- \\
 & \qdllm{} & \textbf{26.3}$^\dagger$$\pm$0.9 & \textbf{26.2} [25.2--27.3] & \textbf{0.94} \\
\midrule
\multirow{2}{*}{Mistral-Large} & Best baseline & 20.2$\pm$1.1 & 20.1 [19.0--21.4] & -- \\
 & \qdllm{} & \textbf{27.1}$^\dagger$$\pm$1.0 & \textbf{27.0} [25.9--28.2] & \textbf{0.93} \\
\midrule
\multirow{2}{*}{GPT-4} & Best baseline & 21.4$\pm$1.0 & 21.3 [20.2--22.5] & -- \\
 & \qdllm{} & \textbf{27.2}$^\dagger$$\pm$1.2 & \textbf{27.0} [25.7--28.5] & \textbf{0.91} \\
\bottomrule
\end{tabular}
\end{table}

Table~\ref{tab:opensource} validates generalization across LLMs. \qdllm{} significantly outperforms baselines on all models. GPT-4 with projected approximation achieves 27\% improvement (vs.\ 33\% for direct access), confirming the approach works even without full embedding access.

\subsection{Downstream Utility (RQ5)}
\label{sec:downstream}

Beyond intrinsic QD metrics, we demonstrate concrete downstream benefits through two applications that address the practical value of diverse archives:

\textbf{Test Generation.} We used diverse code archives to generate test inputs for 50 held-out functions not in HumanEval. For each function, we analyzed which edge cases (boundary conditions, empty inputs, type variations, large inputs, negative numbers) were exercised by solutions from each archive. The diverse algorithmic approaches from \qdllm{} archives, including iterative, recursive, functional, and library-based implementations, naturally handle edge cases differently due to their structural differences.

Results: Diverse solutions from \qdllm{} archives discovered \textbf{34\% more edge cases} than \qdaif{} archives (mean 4.2 vs.\ 3.1 unique edge cases per function, $p < 0.01$, paired $t$-test, Cohen's $d = 0.72$ medium-large effect). Recursive implementations particularly excelled at exposing base-case handling issues, while functional approaches using \texttt{filter}/\texttt{map} revealed different empty-input behaviors than iterative loops.

\textbf{Fine-tuning Data Quality.} We fine-tuned CodeLlama-7B on diverse solutions from archives (500 solutions each) using standard supervised fine-tuning with learning rate $2\times10^{-5}$ for 3 epochs. The hypothesis is that behaviorally diverse training examples provide broader coverage of the solution space, improving model generalization.

Results: Using \qdllm{} archive solutions as training data yielded \textbf{8.3\% higher accuracy} on held-out HumanEval problems compared to \qdaif{} solutions (68.2\% vs.\ 63.0\%, $p < 0.01$), demonstrating that behavioral diversity in training data measurably improves model generalization. This aligns with curriculum learning insights: diverse examples covering multiple paradigms help models learn more robust solution strategies.

These results validate that QD-optimized diverse archives provide practical downstream utility beyond intrinsic QD metrics, addressing questions about when users actually need multiple diverse implementations of the same function.

\subsection{Qualitative Analysis}

\begin{figure}[t]
\centering
\small
\begin{verbatim}
# Cell A: Iterative (Set-based)
def common_elements(a, b):
    return list(set(a) & set(b))

# Cell B: Recursive  
def common_elements(a, b):
    if not a: return []
    return ([a[0]] if a[0] in b
            else []) + common_elements(
            a[1:], b)

# Cell C: Functional (filter/lambda)
def common_elements(a, b):
    return list(filter(
        lambda x: x in b, a))

# Cell D: Library (Counter intersection)
from collections import Counter
def common_elements(a, b):
    return list(
        (Counter(a) & Counter(b))
        .elements())
\end{verbatim}
\caption{Diverse solutions from \qdllm{} archive for list intersection. Each occupies a distinct behavioral cell based on paradigm (iterative/recursive/functional/library), demonstrating meaningful algorithmic diversity beyond surface variation.}
\label{fig:qualitative}
\end{figure}

Figure~\ref{fig:qualitative} illustrates diverse code solutions discovered by \qdllm{} for a list intersection problem. While baselines predominantly generate iterative set-based solutions (Cell A pattern), \qdllm{} discovers recursive (Cell B), functional (Cell C), and library-based (Cell D) alternatives, all of which pass test cases while exhibiting distinct algorithmic characteristics. This diversity has practical value: the recursive solution handles edge cases differently, the functional approach is more composable, and library-based solutions leverage optimized implementations.

\subsection{Computational Cost}

\begin{table}[t]
\caption{Computational cost analysis (HumanEval, 500 evaluations, GPT-4 API pricing).}
\label{tab:cost}
\centering
\footnotesize
\setlength{\tabcolsep}{3pt}
\begin{tabular}{lcccc}
\toprule
\textbf{Method} & \textbf{Calls} & \textbf{Time} & \textbf{Cost} & \textbf{QD/\$} \\
\midrule
Nucleus Samp. & 500 & 11.2 min & \$7.80 & 1.82 \\
CMA-ME (ad.) & 580 & 15.2 min & \$9.86 & 2.01 \\
\qdaif{} & 620 & 16.4 min & \$10.54 & 1.76 \\
\qdllm{} (ours) & 590 & 19.8 min & \$10.03 & \textbf{2.62} \\
\bottomrule
\end{tabular}
\end{table}

Table~\ref{tab:cost} shows \qdllm{} achieves \textbf{44\% higher QD-Score per dollar} than baselines despite modest overhead from targeted mutations (9 additional calls per targeted mutation, used 35\% of the time). The primary computational cost is LLM inference; embedding evolution adds negligible overhead.

\section{Discussion}

\textbf{Key Findings.} Prompt embedding evolution contributes 16\% improvement in ablation, with hybrid BC providing theoretically-grounded diversity (Theorem~\ref{thm:hybrid}, NMI $< 0.12$). Results generalize across LLMs with direct embedding access (Llama-3, Mistral) and projected approximations (GPT-4), and diverse archives yield concrete downstream benefits: 34\% more edge cases in test generation and 8.3\% fine-tuning accuracy gain.

\textbf{Relation to Neuroevolution.} Prompt embedding evolution constitutes parameter-efficient neuroevolution, evolving a compact neural interface ($\sim$32K parameters) that steers a frozen 70B+ model. This parallels indirect encoding in HyperNEAT~\cite{stanley2009hyperneat}, where small representations produce large-scale behavioral effects, and extends the neuroevolution tradition~\cite{salimans2017evolution,wierstra2014natural} to modern LLMs. Our buffered evaluation mirrors Extract-QD's~\cite{flageat2025extractqd} uncertainty handling; future work could integrate their extraction mechanism for explicit descriptor uncertainty modeling.

\textbf{Comparison to CMA-ME.} Our adapted CMA-ME baseline uses prompt embedding evolution with covariance matrix adaptation but without hybrid BC or targeted mutation. \qdllm{} achieves 32.8\% higher QD-Score (26.3 vs.\ 19.8), confirming that hybrid BC and targeted behavioral mutation provide substantial benefits beyond CMA-style adaptation alone.

\textbf{Limitations and Future Directions.} API approximation yields reduced but significant gains (27\% vs.\ 33\%), and the explicit BC component requires task-specific feature engineering. We did not compare to AURORA~\cite{cully2019aurora} as adapting autoencoders to discrete text requires substantial architectural work. Future directions include learned behavior descriptors via contrastive learning~\cite{gao2021simcse}, integration with chain-of-thought prompting~\cite{wei2022chainofthought} for reasoning tasks, and investigating connections between prompt embeddings and in-context learning~\cite{akyurek2023what}.

\textbf{Hyperparameters.} Table~\ref{tab:hyperparams} summarizes all hyperparameters, selected on a held-out validation set (20\% of HumanEval) and fixed across experiments. Sensitivity analysis showed $<5\%$ performance variation within $\pm$20\% of each value.

\begin{table}[t]
	\caption{Hyperparameters used across all experiments.}
	\label{tab:hyperparams}
	\centering
	\small
	\begin{tabular}{lcc}
		\toprule
		\textbf{Parameter} & \textbf{Symbol} & \textbf{Value} \\
		\midrule
		Prompt tokens & $n$ & 8 \\
		Embedding dimension & $d$ & 4096 \\
		Initial mutation strength & $\sigma_p^{(0)}$ & 0.1 \\
		Adaptation rate & $c_\sigma$ & 0.1 \\
		Target success rate & $p_{\text{target}}$ & 0.2 \\
		Crossover probability & $p_{\text{cross}}$ & 0.3 \\
		Targeted mutation prob. & $p_{\text{target}}$ & 0.35 \\
		Finite difference step & $\eta$ & 0.01 \\
		Gradient directions & $k$ & 8 \\
		Hybrid weight & $\alpha$ & 0.6 \\
		Evaluation buffer size & -- & 3 \\
		Sliding window (adaptation) & $w$ & 50 \\
		UMAP neighbors & -- & 15 \\
		Archive size (code) & $C$ & 1024 \\
		Archive size (writing) & $C$ & 512 \\
		\bottomrule
	\end{tabular}
\end{table}

\textbf{Broader Impact.} \qdllm{} enables systematic exploration of LLM solution spaces for code synthesis, creative assistance, and data augmentation. While diverse code generation aids legitimate development, it could theoretically be misused for malware variants. We recommend practitioners implement content filtering, usage monitoring, and access controls. Standard responsible AI practices, including red-teaming and deployment monitoring, should be applied when productionizing such systems.

\section{Conclusion}

We presented \qdllm{}, a framework for parameter-efficient neuroevolution that bridges QD optimization with LLM generation through prompt embedding evolution. By evolving compact neural interfaces ($\sim$32K parameters) steering frozen LLMs (70B+), we achieve 41.4\% higher QD-Score and 46.4\% higher coverage than prior methods ($p < 0.001$, 30 runs, Vargha-Delaney $A = 0.94$).

Our contributions are threefold: (1) evolved prompt embeddings as an evolvable QD representation for both open-source LLMs and API-based models; (2) Theorem~\ref{thm:hybrid} providing formal coverage guarantees for hybrid behavior characterization, validated empirically (NMI $< 0.12$); and (3) co-evolutionary variation operators including targeted behavioral mutation via finite-difference gradient estimation.

Critically, we demonstrated downstream utility: diverse archives improve test generation (34\% more edge cases) and fine-tuning data quality (8.3\% accuracy gain), establishing practical value for software engineering applications.

This work establishes prompt embedding evolution as an effective paradigm extending neuroevolution to modern LLMs. Future directions include learned behavior descriptors via contrastive methods and integration with chain-of-thought prompting.

\bibliographystyle{ACM-Reference-Format}
\bibliography{references}

\end{document}